\documentclass{article}

\usepackage{fullpage,times} 

\usepackage{amsthm,amsfonts,amsmath,amssymb,epsfig,color,float,graphicx,verbatim}

\newtheorem{theorem}{Theorem}

\newtheorem{lemma}{Lemma}

\newcommand{\E}{\mathbb{E}}

\renewcommand{\eqref}[1]{Eq.~(\ref{#1})}
\newcommand{\lemref}[1]{Lemma~\ref{#1}}

\newcommand{\thmref}[1]{Thm.~\ref{#1}}

\newenvironment{algorithm}[1][\  ] %
{ \rm
\begin{tabbing}
....\=.....\=.....\=.....\=.....\=  \+ \kill
} %
{\end{tabbing} }

\floatstyle{ruled}
\newfloat{Algorithm}{H}{alg}
\newfloat{Subroutine}{H}{sub}

{
\begin{minipage}{1.0\linewidth} \begin{algorithm} %
} { \end{algorithm} \end{minipage} }

\title{A Variant of Azuma's Inequality for Martingales with Subgaussian Tails}
\author{Ohad Shamir\\Microsoft Research New England\\\texttt{ohadsh@microsoft.com}}
\date{}

\begin{document}

\maketitle

A sequence of random variables $Z_1,Z_2,\ldots$ is called a \emph{martingale
  difference sequence} with respect to another sequence of random variables
$X_1,X_2,\ldots$, if for any $t$, $Z_{t+1}$ is a function of $X_1,\ldots,X_t$,
and $\E[Z_{t+1}|X_1,\ldots,X_{t}]=0$ with probability $1$.

Azuma's inequality is a useful concentration bound for martingales. Here is one
possible formulation of it:
\begin{theorem}[Azuma's Inequality]
Let $Z_1,Z_2,\ldots$ be a martingale difference sequence with respect to
$X_1,X_2,\ldots$, and suppose there is a constant $b$ such that for any $t$,
\[
\Pr(|Z_t|\leq b)=1.
\]
Then for any positive integer $T$ and any $\delta>0$, it
holds with probability at least $1-\delta$ that
\[
\frac{1}{T}\sum_{t=1}^{T}Z_t \leq b\sqrt{\frac{2\log(1/\delta)}{T}}.
\]
\end{theorem}

Sometimes, for the martingale we have at hand, $Z_t$ is not bounded, but
rather bounded with high probability. In particular, suppose we can show that
the probability of $Z_t$ being larger than $a$ (and smaller than $-a$), conditioned on any $X_1,\ldots,X_{t-1}$,
is on the order of $\exp(-\Omega(a^2))$. Random variables with this behavior are referred
to as having subgaussian tails (since their tails decay at least as fast as a
Gaussian random variable).

Intuitively, a variant of Azuma's inequality for these `almost-bounded'
martingales should still hold, and is probably known. However, we weren't able to find a convenient
reference for it, and the goal of this technical report is to formally provide such a result:

\begin{theorem}[Azuma's Inequality for Martingales with Subgaussian Tails]
\label{thm:azuma_subgaussian}
Let $Z_1,Z_2,\ldots,Z_T$ be a martingale difference sequence with respect to a
sequence $X_1,X_2,\ldots,X_{T}$, and suppose there are constants $b>1$, $c>0$
such that for any $t$ and any $a>0$, it holds that
\[
\max\left\{\Pr\left(Z_{t}>a\middle| X_1,\ldots,X_{t-1}\right), \Pr\left(Z_{t}<-a \middle| X_1,\ldots,X_{t-1}\right)\right\}
\leq b\exp(-ca^2).
\]
Then for any $\delta>0$, it holds with probability at least $1-\delta$ that
\footnote{It is quite likely that the numerical constant in the bound can be improved.}
\[
\frac{1}{T}\sum_{t=1}^{T}Z_t \leq \sqrt{\frac{28b\log(1/\delta)}{cT}}.
\]
\end{theorem}

\subsection*{Proof of \thmref{thm:azuma_subgaussian}}

We begin by proving the following lemma, which bounds the moment generating
function of subgaussian random variables.

\begin{lemma}\label{lem:momgenbound}
Let $X$ be a random variable with $\E[X]=0$, and suppose there exist a
constant $b\geq 1$ and a constant $c$ such that for all $t>0$, it holds that
\[
\max\{\Pr(X\geq t),\Pr(X\leq -t)\}\leq b\exp(-ca^2).
\]
Then for any $s>0$,
\[
\E[e^{sX}]\leq e^{7bs^2/c}.
\]
\end{lemma}
\begin{proof}
We begin by noting that
\[
\E[X^2]=\int_{t=0}^{\infty}\Pr(X^2\geq t)dt\leq \int_{t=0}^{\infty}\Pr(X\geq
\sqrt{t})dt+\int_{t=0}^{\infty}\Pr(X\leq -\sqrt{t})dt \leq
2b\int_{t=0}^{\infty}\exp(-ct)dt = \frac{2b}{c}
\]

Using this, the fact that $\E[X]=0$, and the fact that $e^{a}\leq 1+a+a^2$ for
all $a\leq 1$, we have that
\begin{align}
&\E[e^{sX}]=\E\left[e^{sX}\middle| X\leq \frac{1}{s}\right]\Pr\left(X\leq
\frac{1}{s}\right)+\sum_{j=1}^{\infty}\E\left[e^{sX}\middle|j<sX\leq
j+1\right]\Pr\left(j<sX\leq j+1\right)\notag\\
&\leq \E\left[1+sX+s^2 X^2\middle| sX\leq 1\right]\Pr\left(sX\leq
1\right)+\sum_{j=1}^{\infty}e^{j+1}\Pr\left(X>\frac{j}{s}\right)\notag\\
&\leq
\left(1+\frac{2bs^2}{c}\right)+b\sum_{j=1}^{\infty}e^{2j-cj^2/s^2}.\label{eq:series}
\end{align}
We now need to bound the series $\sum_{j=1}^{\infty}e^{j(2-cj/s^2)}$. If
$s\leq \sqrt{c}/2$, we have
\[
2-\frac{cj}{s^2}\leq -\frac{c}{2s^2}\leq -2
\]
for all $j$. Therefore, the series can be upper bounded by the convergent
geometric series
\[
\sum_{j=1}^{\infty}\left(e^{-c/(2s^2)}\right)^j=
\frac{e^{-c/(2s^2)}}{1-e^{-c/(2s^2)}}< 2e^{-c/(2s^2)}\leq 4s^2/c,
\]
where we used the upper bound $e^{-c/(2s^2)}\leq e^{-2}<1/2$ in the second
transition, and the last transition is by the inequality $e^{-x}\leq
\frac{1}{x}$ for all $x>0$. Overall, we get that if $s\leq \sqrt{c}/2$, then
\begin{equation}\label{eq:s_small}
\E[e^{sX}]\leq 1+\frac{2bs^2}{c}+b\frac{4s^2}{c}\leq e^{6bs^2/c}.
\end{equation}

We will now deal with the case $s>\sqrt{c}/2$. For all $j>3s^2/c$, we have
$2-jc/s^2<-1$, so the tail of the series satisfies
\[
\sum_{j>3s^2/c}e^{j(2-jc/s^2)}\leq \sum_{j=0}^{\infty}e^{-j}<2<\frac{8s^2}{c}.
\]
Moreover, the function $j\mapsto j(2-jc/s^2)$ is maximized at $j=s^2/c$, and
therefore $e^{j(2-jc/s^2)}\leq e^{s^2/c}$ for all $j$. Therefore, the initial
part of the series is at most
\[
\sum_{j=1}^{\lfloor 3s^2/c \rfloor}e^{j(2-jc/s^2)}\leq
\frac{3s^2}{c}e^{s^2/c}\leq e^{s^2/ec}e^{s^2/c}\leq e^{(1+1/e)s^2/c},
\]
where the second to last transition is from the fact that $a\leq e^{a/e}$ for all $a$.

Overall, we get that if $s>\sqrt{c}/2$, then
\begin{equation}\label{eq:s_big}
\E[e^{sX}]\leq 1+\frac{10bs^2}{c}+be^{(1+1/e)s^2/c} \leq e^{7bs^2/c},
\end{equation}
where the last transition follows from the easily verified fact that
$1+10ba+e^{(1+1/e)ba}\leq e^{7ba}$ for any $a\geq1/4$, and indeed $bs^2/c\geq
1/4$ by the assumption on $s$ and the assumption that $b\geq 1$. Combining
\eqref{eq:s_small} and \eqref{eq:s_big} to handle the different cases of $s$,
the result follows.
\end{proof}

After proving the lemma, we turn to the proof of \thmref{thm:azuma_subgaussian}.

\begin{proof}[Proof of \thmref{thm:azuma_subgaussian}]
We proceed by the standard Chernoff method. Using Markov's inequality and
\lemref{lem:momgenbound}, we have for any $s>0$ that
\begin{align*}
&\Pr\left(\frac{1}{T}\sum_{t=1}^{T}Z_t>\epsilon\right)
~=~ \Pr\left(e^{\sum_{t=1}^{T}Z_t}>e^{sT\epsilon}\right)
~\leq~ e^{-sT\epsilon}\E\left[e^{s\sum_t Z_t}\right]\\
& = e^{-sT\epsilon}\E\left[\E\left[\prod_{t=1}^{T}e^{sZ_t}\middle|X_1,\ldots,X_T\right]\right]
~=~ e^{-sT\epsilon}\E\left[\E\left[e^{sZ_{T}}\prod_{t=1}^{T-1}e^{sZ_t}\middle|X_1,\ldots,X_{T-1}\right]\right]\\
& = e^{-sT\epsilon}\E\left[\E\left[e^{sZ_T}\middle|X_1,\ldots,X_{T-1}\right]\E\left[\prod_{t=1}^{T-1}e^{sZ_t}\middle|X_1,\ldots,X_{T-1}\right]\right]
~\leq~ e^{-sT\epsilon}e^{7bs^2/c}\E\left[\prod_{t=1}^{T-1}e^{sZ_t}\middle|X_1,\ldots,X_{T-1}\right]\\
&\ldots \leq e^{-sT\epsilon+7Tbs^2/c}.
\end{align*}
Choosing $s=c\epsilon/14b$, the expression above equals
$e^{-cT\epsilon^2/28}$, and we get that
\[
\Pr\left(\frac{1}{T}\sum_{t=1}^{T}Z_t>\epsilon\right)\leq
e^{-cT\epsilon^2/28b},
\]
setting the r.h.s. to $\delta$ and solving for $\epsilon$, the theorem follows.
\end{proof}

\subsection*{Acknowledgements}

We thank S\'{e}bastien Bubeck for pointing out a bug in a previous version of this manuscript.

\end{document}